\documentclass[10pt,twocolumn,letterpaper]{article}

\usepackage{cvpr}
\usepackage{times}
\usepackage{epsfig}
\usepackage{graphicx}
\usepackage{amsmath}
\usepackage{amssymb,amsthm}
\usepackage{setspace}

\usepackage{array}
\usepackage{algorithm}
\usepackage{enumerate}
\usepackage[titletoc,title]{appendix}

\newtheorem{theorem}{Theorem}

\newtheorem{proposition}[theorem]{Proposition}

\def\WidthOne{1.65cm}

\usepackage[pagebackref=true,breaklinks=true,letterpaper=true,colorlinks,bookmarks=false]{hyperref}

\cvprfinalcopy 

\def\cvprPaperID{1786} 

\ifcvprfinal\fi
\begin{document}

\title{Spectral Clustering  with Jensen-type kernels and their multi-point extensions}
\author{Debarghya Ghoshdastidar,
Ambedkar Dukkipati, Ajay P. Adsul and Aparna S. Vijayan\\ 
Department of Computer Science \& Automation\\
Indian Institute of Science\\
Bangalore 560012, India\\
\texttt{\small \{debarghya.g,ad,ajay.adsul,aparna\}@csa.iisc.ernet.in} }

\maketitle

\begin{abstract}
Motivated by multi-distribution divergences,
which originate in information theory, we propose a
notion of `multi-point' kernels, and study
their applications. We study a class of kernels based on
Jensen type divergences and show that 
these can be extended to measure similarity among multiple points.
 We study tensor flattening methods 
 and develop a multi-point (kernel) spectral clustering (MSC) method.
 We further emphasize on a special
 case of the proposed kernels, which is a multi-point extension of the
 linear (dot-product) kernel and show the existence of cubic time
 tensor flattening algorithm in this case. Finally, we illustrate the
 usefulness of our contributions using standard data sets and image
 segmentation tasks.  
\end{abstract}
\section{Introduction}
Divergences, though introduced at the birth of information theory,
Jensen-Shannon (JS) divergence appeared in the literature relatively
recently~\cite{Lin_1991_jour_TIT}, and 
the unique characteristic of this divergence is that one can measure a divergence between more than
two probability distributions. Hence, one can term this as a multi-distribution divergence. 


Recently, there have been a growing interest in kernel connections of 
the JS divergence in the machine learning community that started
from the works of Endres and Schindelin~\cite{Endres_2003_jour_TIT},
who observed that 
$\sqrt{\text{JS}}$ is a Hilbertian metric. This renewed interests in
viewing Jensen-type divergences as dissimilarity measures.
Studies by Martins et al.~\cite{Martins_2009_jour_JMLR} extend the idea further to 
the \emph{nonextensive} case to formulate the so-called Jensen-Tsallis (JT) kernels on finite measures,
that has proved to be quite useful in 
text classification~\cite{Martins_2009_jour_JMLR} and shape recognition~\cite{Bicego_2010_conf_ICPR}.

Though the JT-kernels and their applications have been well studied,
two significant implications of these kernels have not been explored yet.
The first one lies in the simple observation that JT-kernel
retrieves the linear/dot-product kernel ($x^Ty$) in a special case,
and hence, these kernels may have interesting properties even on the Euclidian space.
The second is multi-distribution nature of JS-divergence.
This fact easily extends to Jensen-type kernels, in particular the JT-kernel,
which leads to the notion of multi-point kernels studied in this paper.

The concept of multi-point similarities can be traced back to 
the studies on $n$-metrics, which started during works of Hayashi~\cite{Hayashi_1972_jour_AISM}.
But its applications to unsupervised learning has been observed relatively 
recently~\cite{Govindu_2005_conf_CVPR,Agarwal_2005_conf_CVPR}.
The works of Govindu~\cite{Govindu_2005_conf_CVPR} and 
Chen and Lerman~\cite{Chen_2009_jour_JourCV}
is worth mentioning in this respect, who combined multi-point similarities to 
spectral clustering in context of computer vision.
However, the proposed methods are model specific, and hence, are restricted to 
applications like hybrid linear modeling and motion segmentation.
On the other hand, spectral methods~\cite{Shi_2000_jour_TPAMI,Ng_2002_conf_NIPS}
are quite general and their scope is broad ranging from 
image segmentation~\cite{Shi_2000_jour_TPAMI} to analysis of
correlated mutations in HIV-1
protease~\cite{Liu_2008_jour_Bioinformatics}. 
To cater to the widely varying application of spectral clustering and to improve upon
the Gaussian distance measure commonly employed in spectral based learning,
it is quite tempting to study the spectral clustering using multi-point kernels
as done in this work. Our approach distinctly deviates from the existing 
multi-point spectral methods in the use of multi-point kernel not restricted
to model dependent similarities as in spectral curvature clustering (SCC)~\cite{Chen_2009_jour_JourCV}.
The contributions in this paper are listed below:
\\(1) We extend the JT-kernels on finite measures 
 to define similar kernels on the $d$-dimensional unit cube $[0,1]^d$,
 that encompass the linear (dot-product) kernel.
 Further, we use the idea of multi-distribution divergences to define multi-point 
 extensions of above kernels.
\\(2) We develop a model-independent spectral clustering algorithm using multi-point kernels,
 which we call as MSC.
 Though MSC has an exponential complexity, like SCC~\cite{Chen_2009_jour_JourCV}, we prove that 
 a cubic time complexity can be achieved in the 
 special case of multi-point extension of linear kernel.
\\(3) We study the performance of the proposed method and the kernels
 in the context of image segmentation.

\section{Nonextensive Jensen-type kernels}
\label{preliminaries}
Before going into the main discussions of this paper, 
we briefly review the JT-kernels on probability measures~\cite{Martins_2009_jour_JMLR}.
It suffices to study kernels on
the $d$-dimensional probability simplex, denoted by
$\Delta^{d-1} = \{ (p(1),\ldots,p(d)): p(j) \in [0,1] \forall j, \sum_{j=1}^{d} p(i) = 1 \}$.
The Jensen-Tsallis $q$-difference among $n$ p.m.f.s 
{$p_{i} = (p_{i}(1),\ldots, p_i(d)) \in \Delta^{d-1}$}, {$i=1,\ldots,n$} 
is defined as~\cite[Section 5]{Martins_2009_jour_JMLR} 
\begin{align}
T_q\left(p_{1}, \ldots, p_{n} \right) &= H_q\left(\bar{p}\right)- \frac{1}{n^q} \sum_{i=1}^n H_q(p_{i}),
\label{eq_jt_qdiff}
\end{align}
where $\bar{p}=(\bar{p}(1),\ldots,\bar{p}(d))$ is the p.m.f. defined as 
{$\bar{p}(i) = \frac{1}{n}\sum_{j=1}^{n} p_{j}(i)$}, $i=1,\ldots,d$, and
$H_q$ is the nonextensive or Tsallis entropy~\cite{Tsallis_1988_jour_StatPhy}
that has been extensively used in statistical mechanics to study multifractal concepts.
It is given by  $H_q(p) = \frac{1}{(q-1)} \big({1 - \sum_{j=1}^{d} p(j)^q} \big)$,
where $q\in\mathbb{R}$, $q\neq1$ is a parameter related to the nature of the physical system.
As $q\to1$, the classical case of Shannon entropy  
is retrieved, $H_1(p) = - \sum_{j=1}^{d} p(i) \ln \big({p(j)}\big)$, and the $q$-difference~\eqref{eq_jt_qdiff}
in this case corresponds to the JS-divergence.

Martins et al.~\cite{Martins_2009_jour_JMLR} showed the theoretical justifications behind
the definition of JT $q$-difference, and observed that {$T_q\left(p_{1}, p_{2}\right) \leqslant \ln_q(2)$}
for all {$p_{1}, p_{2} \in \Delta^{d-1}$}. Based on this, a kernel on 
probability measures {$\widetilde{k}_q:\Delta^{d-1}\times\Delta^{d-1}\mapsto[0,\infty)$}
was proposed as~\cite[Definition 26]{Martins_2009_jour_JMLR}
\begin{align}
  &\widetilde{k}_q\left(p_{1}, p_{2}\right) = 2^q\left(\ln_q(2) - T_q\left(p_{1}, p_{2}\right)\right) 
  \nonumber\\
  &= {\textstyle\frac{1}{(q-1)}} \sum_{j=1}^{d} \big(\left(p_{1}(j) + p_{2}(j)\right)^q - p_{1}(j)^q - p_{2}(j)^q\big)
\label{eq_jt_ker_prob}
\end{align}
for $q\neq1$, which is the Jensen-Tsallis (JT) kernel between 
the two probability measures {$p_{1}$} and {$p_{2}$}.
The above class of kernels {$\widetilde{k}_q$} is positive definite on $\Delta^{d-1}$
for $0\leqslant q \leqslant 2$~\cite{Martins_2009_jour_JMLR}.
For $q=2$, we have a dot-product kernel on $\Delta^{d-1}$
\begin{align}
  \label{eq_lin_ker_prob}
  \widetilde{k}_2\left(p_{1}, p_{2}\right)
  = 2\sum_{j=1}^{d} p_{1}(j)p_{2}(j) = 2\left(p_{1}(j)\right)^{T}\left(p_{2}(j)\right),
\end{align}
and in the limit of $q\to1$, we have the JS-kernel defined as
\begin{align}
  \widetilde{k}_1\left(p_{1}, p_{2}\right)
  =& \sum_{j=1}^{d} \big(\left(p_{1}(j) + p_{2}(j)\right)\ln\left(p_{1}(j) + p_{2}(j)\right)
  \nonumber \\
  &- p_{1}(j)\ln\left(p_{1}(i)\right) - p_{2}(j)\ln\left(p_{2}(j)\right)\big).
  \label{eq_js_ker_prob}
\end{align}

\section{The notion of multi-point kernels}
\label{sec_real_extn}
We now present the main idea of this paper -- multi-point extensions of the 
JT-kernels,~\eqref{eq_jt_ker_prob}-\eqref{eq_js_ker_prob}. To extend the scope
of these kernels, we first extend them to the real space.
More specifically, we present extensions to the set $[0,1]^d$. 
This is a technical requirement, and is not restrictive since
it is a common practice to normalize features of data
and such a set suffices for most datasets.
We proceed along the lines
of the defined probability kernel~\eqref{eq_jt_ker_prob}, and define
an extension of JT-kernel {$k_q:[0,1]^d\times[0,1]^d\mapsto[0,\infty)$} of the form
\begin{align}
  &{k}_q(x, y) = 
  \nonumber \\
  &\left\{
      \begin{array}{l}
        \frac{1}{(q-1)} \sum\limits_{j=1}^{d} \big( (x(j) + y(j))^q - x(j)^q - y(j)^q \big) 
        \\ 
        \hfill \text{for } q\neq1 
	\\
        \sum\limits_{j=1}^{d} \big( (x(j) + y(j))\ln(x(j) + y(j)) 
        \\
        \quad - x(j)\ln(x(j)) - y(j)\ln(y(j)) \big) 
        \hfill \text{ for } q=1,
      \end{array} \right.
\label{eq_jt_ker}
\end{align}
where {$x=(x(1),\ldots,x(d)),y=(y(1),\ldots,y(d))\in[0,1]^d$}. 
The special case of linear kernel on $[0,1]^d$ follows similar to~\eqref{eq_lin_ker_prob}.
The significance of JT-kernel is the fact that
while the Gaussian kernel follows the nature of the Euclidian distance, similar to the linear kernel,
the distance in case of Jensen-type kernels usually exhibit a skewed behavior.
Further, localization effects are much less in the Jensen-type kernels 
as compared to the Gaussian kernel, since they are not exponentially decaying.
The following result shows that the
above extension does not affect the positive definiteness of the kernel.
This can be proved by mimicking
the proof of \cite[Proposition 27]{Martins_2009_jour_JMLR}
using the above kernel function.
\begin{proposition}
\label{thm_2pt_pd}
     JT-kernels {$k_q$} are positive definite on $[0,1]^d$ for all dimensions $d$ and all $q\in[0,2]$. 
\end{proposition}
%

We now present the multi-point extensions of the JT-kernel~\eqref{eq_jt_ker}.
The idea is based on the multi-distribution definition of Jensen-Tsallis $q$-difference~\eqref{eq_jt_qdiff}
where $n$ need not be equal to 2.
We extend the JT-kernel for arbitrary number of points in $\mathcal{X}=[0,1]^{d}$ to obtain a 
class of multi-point kernels $\{K_{q,n}\}_{n\in\mathbb{N}}$ with
{$K_{q,n}: \mathcal{X}^n \mapsto[0,\infty)$} defined as
\begin{align}
  &K_{q,n}\left(x_{1},\dots,x_{n}\right) =
  \nonumber\\
  &\left\{
      \begin{array}{ll}
        \frac{1}{(q-1)} \displaystyle\sum_{j=1}^{d}\bigg[ \bigg( \sum_{i=1}^{n}x_{i}(j)\bigg)^{q} 
        - \sum_{i=1}^{n}\left(x_{i}(j)\right)^{q} \bigg] 
        \hfill\text{ for } q\neq1\\
	\\
        \displaystyle\sum_{j=1}^{d} \bigg[\bigg(\sum_{i=1}^{n}x_{i}(j)\bigg) \ln \bigg(\sum_{i=1}^{n}x_{i}(j)\bigg)
        \\
        \hfill - \displaystyle\sum_{i=1}^{n}x_{i}(j)\ln x_{i}(j)\bigg] \text{ for } q=1.\\
      \end{array} \right.
\label{eq_jt_ker_mul}
\end{align}
The above definition is consistent with the multi-distribution extensions of JT $q$-difference.
Since it naturally extends a positive definite kernel, we refer to it as a kernel.  
In the linear case, \textit{i.e.}, for $q=2$, we retrieve a multi-point version of the dot-product kernel as
\begin{equation}
  K_{2,n}(x_{1},\ldots ,x_{n}) = 2 \sum_{i=1}^{n}\sum_{j=i+1}^{n}{x_{i}^{T}x_{j}} \;,
  \label{eq_lin_ker_mul}
\end{equation}
which will be discussed in greater detail in sequel.
%
The above extension of two-point kernels captures information about similarity 
among multiple points, and is capable of providing a more global measure of similarity.
Further, the proposed multi-point similarity is not dependent 
on any geometric model, unlike the ones in~\cite{Govindu_2005_conf_CVPR,Chen_2009_jour_JourCV},
and hence, it is applicable in a more general framework. 
Next, we present a spectral clustering method based on multi-point kernels.
The basic approach is similar to the 
spectral curvature clustering (SCC)~\cite{Chen_2009_jour_JourCV},
but it is applicable for any multi-point kernel.

\section{Multi-point spectral clustering}
\label{sec_spectral}

\subsection{Algorithm}
\begin{algorithm}[b]
\caption{Multi-point Spectral Clustering (MSC)}
\label{algo_SCC}
{\bf Given: } $n^{th}$ order tensor $\mathcal{A}$ representing affinity among data points {$\{x_{1},\ldots,x_{N}\} \in \mathcal{X}$}.

{\small\bf 1.} Unfold $\mathcal{A}$ to obtain flattened matrix $A$, and let $V = AA^{T}$.\\
{\small\bf 2.} Normalize affinity matrix as $Z = D^{-1/2}VD^{-1/2}$, where $D$ is a diagonal matrix with $d_{ii} = \sum_{j=1}^{N} V_{ij}$. \\
{\small\bf 3.} Compute $u_1,\ldots,u_m$, top-$m$ unit eigenvectors of $Z$. \\
{\small\bf 4.} Normalize rows of $U = [u_1,\ldots,u_m]$ to have unit length.\\
{\small\bf 5.} Cluster the rows of $U$ into $m$ clusters using $k$-means, and partition $\{x_{1},\ldots,x_{N}\}$ accordingly.\\
\end{algorithm}

We consider the problem of clustering $N$ points, {$\{x_{1},\ldots,x_{N}\} \in \mathcal{X}$},
into $m$ clusters, $C_1,\ldots,C_m$,
using any $n$-point similarity measure $K:\mathcal{X}^n\mapsto\mathbb{R}$.
The similarity among different points is represented by a $n^{th}$ order $N$-dimensional real tensor
$\mathcal{A}$, where $\mathcal{A}_{i_1,i_2,\ldots,i_n} = K\left( x_{i_1},x_{i_2},\ldots,x_{i_n} \right)$
for $i_j = 1,\ldots,N$ with $j=1,\ldots,n$. 
We observe from~\eqref{eq_jt_ker_mul} that
$K$ is permutation invariant, \textit{i.e.}, the similarity does not change if the arguments are re-ordered.
Hence, the tensor $\mathcal{A}$ is super-symmetric.
The idea is to construct a similarity (or, affinity) matrix from $\mathcal{A}$.
This is done by tensor unfolding or mode-1 matricization~\cite{Lathauwer_2000_jour_SIAM},
where we construct a matrix $A \in\mathbb{R}^{N\times N^{n-1}}$ whose 
$j^{th}$ column, for $j= 1 + \sum_{{l=2}}^{n}(i_{l}-1) N^{l-1}$
is the stack of tensor $\mathcal{A}$ obtained by
varying the first index, and fixing others at $(i_2,\ldots,i_n)$.
From $A$, the affinity matrix is constructed as $V=AA^{T}$
that preserves the left eigenvectors of $A$ (or mode-1 eigenvectors of $\mathcal{A}$).
Below, we state the algorithm based on spectral clustering algorithm due to Ng et al.~\cite{Ng_2002_conf_NIPS}.

The complexity of MSC is quite large since
computation of each element in $V$ requires $2N^{n-1}$ kernel computations,
and hence, complexity of determining $V$ turns out to be $O(N^{n+1})$.
We can incorporate the heuristic approach mentioned in~\cite{Govindu_2005_conf_CVPR},
to approximate $V$ as $V \approx \sum_{k=1}^{c} w_{j_k}w_{j_k}^{T}$,
by uniformly sampling $c$ columns from all the $N^{n-1}$ columns of $A$, 
where $w_{j_k}$ denotes the ${j_k}^{th}$ column of $A$.
Though computation reduces to a great extent to $O(cN^2)$ for $c\ll N^{n-1}$,
the performance of the algorithm is quite poor 
in general, when model underlying the data is not known a priori.
In fact, since MSC does not assume geometric structures, the effect of such 
approximations is quite severe in this case.
More efficient methods discussed in~\cite{Chen_2009_jour_JourCV} in context of SCC
can be used.
We do not discuss such approximations here,
but focus on a special case of multi-point JT-kernel,
where cubic time complexity is achieved for MSC.

\subsection{MSC using multi-point linear kernel}
\label{sec_multi_linear}

Recall that the multi-point JT-kernel for $q=2$~\eqref{eq_lin_ker_mul},
which is a multi-point extension of linear kernel.
The structure of this multi-point linear kernel helps 
to compute the affinity matrix $V$ explicitly in cubic time as shown below.
\begin{proposition}
\label{thm_mul_lin_affinity}
 Let $X = (x_1,x_2,\ldots, x_N) \in [0,1]^{d\times N}$ represent the given data matrix and 
 $\bar{x} := \sum_{i=1}^{N} x_{i}$ be the component-wise addition of the vectors.
 Then, the affinity matrix $V$ corresponding 
 to the $n$-point linear kernel $K_{2,n}$~\eqref{eq_lin_ker_mul} can be written as
 \begin{align}
  V &= 4\textstyle\binom{n-1}{1}N^{n-2}\left(X^T X\right)^2 
  + 8\binom{n-1}{2} N^{n-3} \left( X^T \bar{x}\bar{x}^T X \right)
  \nonumber  \\
  &+ 8\textstyle\binom{n-1}{2} N^{n-3} \left( X^T XX^T \bar{x}\mathbf{1}_{1\times N}
  + \mathbf{1}_{N\times1}\bar{x}^T XX^T X \right)
  \nonumber \\
  &+ 12\textstyle\binom{n-1}{3} N^{n-4} \Vert\bar{x}\Vert_2^2 
  \left( X^T \bar{x}\mathbf{1}_{1\times N} + \mathbf{1}_{N\times1} \bar{x}^T X  \right)
  \nonumber \\
  &+ 4\textstyle\binom{n-1}{2} N^{n-5} \Big( N^2 \left\Vert{X^T X}\right\Vert_F^2 
  + 2(n-3)N\left\Vert{X^T\bar{x}}\right\Vert_2^2
  \nonumber \\
  & \qquad\qquad\qquad\qquad\qquad + 2\textstyle\binom{n-3}{2}\Vert\bar{x}\Vert_2^4 \Big)\mathbf{1}_{N\times N}
 \label{eq_lin_V}
 \end{align}
where $\mathbf{1}_{r\times s}$ denotes a $r\times s$ matrix of all 1's,
$\Vert.\Vert_F$ is the Frobenius norm.
\end{proposition}
\begin{proof}
We provide a brief sketch of the proof.
Note that $V=AA^T$ and it can be written as 
\begin{align}
V &= 4 \sum_{i_2,\ldots,i_n = 1}^{N} \Bigg[
X^T \left(\sum_{l=2}^{n} \sum_{r=2}^{n} x_{i_l} x_{i_r}^T \right) X
\nonumber\\
&+ X^T \left( \sum_{l=2}^{n} \sum_{r=2}^{n} \sum_{s=r+1}^{n} x_{i_l} x_{i_r}^T x_{i_s} \right) \mathbf{1}_{1\times N}
\nonumber \\
&+ \mathbf{1}_{1\times N} \left( \sum_{r=2}^{n} \sum_{l=2}^{n} \sum_{k=l+1}^{n} x_{i_r} x_{i_l}^T x_{i_k} \right) X
\nonumber \\
&+ \left( \sum_{l=2}^{n} \sum_{r=2}^{n} \sum_{k=l+1}^{n} \sum_{s=r+1}^{n} x_{i_l}^T x_{i_s} x_{i_r}^T x_{i_s} \right) \mathbf{1}_{N\times N} \Bigg]\;,
\label{eq_thm_lin_V}
\end{align}
where we use the fact that given $i_2,\ldots,i_n$, the $j^{th}$ column
of $A$, where $j= \left(1 + \sum_{l=2}^{n} (i_{l}-1) N^{l-2} \right)$,
is simply
$2X^T \left(\sum_{l=2}^{n} x_{i_l} \right)
+ 2\left(\sum_{l=2}^{n} \sum_{k=l+1}^{n} x_{i_l}^T x_{i_k}\right) \mathbf{1}_{N\times1}$.
Comparing~\eqref{eq_thm_lin_V} and~\eqref{eq_lin_V},
we observe that the first term in~\eqref{eq_thm_lin_V} 
decomposes into the first two terms of~\eqref{eq_lin_V}.
The second and third terms of~\eqref{eq_thm_lin_V} contribute to the third and 
fourth terms of~\eqref{eq_lin_V}, while the last term of~\eqref{eq_thm_lin_V} is equal to 
the last term in~\eqref{eq_lin_V}.
Also, the outer summation in~\eqref{eq_thm_lin_V} may be pushed inside
to simplify the results of the inner summations as shown below.
For the first term, we consider the outer product of same and distinct vectors
separately as
\begin{align}
&\sum_{i_2,\ldots,i_n = 1}^{N} \sum_{l=2}^{n} \sum_{r=2}^{n} x_{i_l} x_{i_r}^T
\label{eq_thm_lin_V1}\\
\nonumber
&= N^{n-2}\sum_{l=2}^{n} \sum_{i_l = 1}^{N} x_{i_l} x_{i_l}^T + N^{n-3}
\sum_{{r,l=2,  r\ne l}}^{n} \sum_{i_l,i_r = 1}^{N} x_{i_l} x_{i_r}^T
\end{align}
since the terms act as constants while summing over all indices other than $i_l$ and $i_r$, and each such summation adds up $N$ similar terms, leading to the constants
outside the summations. 
Now, one can verify that $XX^T = \sum_{i=1}^{N} x_i x_i^T$ and
$\bar{x}\bar{x}^T = \sum_{i,j=1}^{N} x_i x_j^T$.
Plugging this in~\eqref{eq_thm_lin_V1}, and noting that there are $(n-1)$ terms in the
first summation and $2\binom{n-1}{2}$ terms in the second leads to the first two terms of~\eqref{eq_lin_V}.
To deal with the second term of~\eqref{eq_thm_lin_V}, it is enough to show that
\begin{align}
&\sum_{i_2,\ldots,i_n = 1}^{N} \sum_{l=2}^{n} \sum_{r=2}^{n} \sum_{s=r+1}^{n} x_{i_l} x_{i_r}^T x_{i_s} 
\nonumber\\
&= 2\textstyle\binom{n-1}{2}N^{n-3} XX^T \bar{x} 
+ 3\textstyle\binom{n-1}{3} N^{n-4} \Vert\bar{x}\Vert_2^2 \bar{x}\;.
\label{eq_thm_lin_V2}
\end{align}
The constants $N^{n-3}$ and $N^{n-4}$ appear as before due to summation over
indices, which are absent from the terms involved.
We consider the cases $r=l$ and $r\neq l$ separately. For $r=l$,
we obtain half of the first term in~\eqref{eq_thm_lin_V2} since
$\sum_{r=2}^{n} \sum_{s=r+1}^{n} \sum_{i_r,i_s = 1}^{N} x_{i_r} x_{i_r}^T x_{i_s} 
= \textstyle\binom{n-1}{2} XX^T \bar{x}$.
For $r\neq l$, the situation becomes complicated as we may have $s=l$. 
But this happens only in $\binom{n-1}{2}$ cases, which adds up to give
the remaining half of the first term in~\eqref{eq_thm_lin_V2}.
The rest of the terms on the left in~\eqref{eq_thm_lin_V2}
have distinct indices, and hence, summing over them gives a term of the
form $\sum_{i,j,k=1}^{N} x_i x_j^T x_k = \Vert\bar{x}\Vert_2^2 \bar{x}$.
But, there are $3\binom{n-1}{3}$ such terms, and hence, the result.
Similarly, computing the other terms in~\eqref{eq_thm_lin_V}, one can 
derive the expression in~\eqref{eq_lin_V}.
\end{proof}
\begin{table*}
\centering
\caption{Comparison of MSC and Gaussian spectral clustering (for Isolet dataset, we consider only classes A,B,C).}
\label{tab:multicol_spec}
\begin{tabular}{|c||l|l|l|l|}
\hline
Dataset  	& Gaussian SC ($\sigma$)	& 2-point JT kernel ($q$)	& 3-point JT kernel ($q$)	& $n$-point linear kernel ($n$)\\
\hline \hline
Breast Cancer		&0.968 (0.5)			& 0.963 (2.00)			& \textbf{0.971} (1.0)	& 0.966 (6-12)	\\
Isolet (ABC)		&0.863 (10.0)			& \textbf{0.965} (1.25)	& \textbf{0.965} (1.0-1.25)	& 0.929 (4)	\\
Iris			&0.930 (0.15)			& 0.860 (0.0-0.5)		& \textbf{0.965} (0.5)	& 0.792 (10)	\\
Mammographic mass	&0.799 (0.3)			& 0.807 (2.0)			& 0.776 (1.5)			& \textbf{0.810} (4-12)	\\
Semeion hand-written	&\textbf{0.604} (5.0)	& 0.534 (1.25)			& 0.569 (0.25)			& 0.561 (4)	\\
\hline
\end{tabular}
\end{table*}
The key fact in above result is that all computations in~\eqref{eq_lin_V}
are at most $O(N^3)$, which implies that $V$ is computable in cubic time.
Further, though the above result holds for any $n\in\mathbb{N}$, few simplifications 
are possible for $n\leqslant4$. For instance, if $n=2$, all terms vanish except first, giving
$V = 4 (X^T X)^2$,
which has the same eigen structure as $X^T X$. 
Hence, spectral clustering with $V$ is equivalent
to the case of constructing affinity using the Gram matrix.
We illustrate the behavior of the multi-point linear kernel
with a simple example of two concentric arcs in $[0,1]^2$
(Figure~\ref{fig_lin_spectral}).
This is an example where $k$-means algorithm fails.
We use both Gaussian spectral clustering and MSC with $n$-point linear kernel,
and observe that for small $n$ (MSC) and large $\sigma$ (Gaussian)
both methods are quite similar to $k$-means. 
Accurate clustering can be achieved for Gaussian, but this requires
proper tuning of $\sigma$ as we see that even for small variations 
of $\sigma$-values considered, the results vary considerably.
On the other hand, if the large number of points are considered,
MSC gives accurate results. In fact, in this example,
we observed that results improved with increase in $n$, and for $n\geqslant 7$
correct clustering were always achieved.

\begin{figure}[h]
\centering
\includegraphics[width=0.4\textwidth]{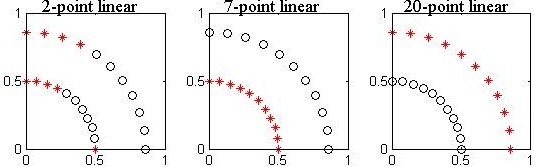}
\\
~\includegraphics[width=0.4\textwidth]{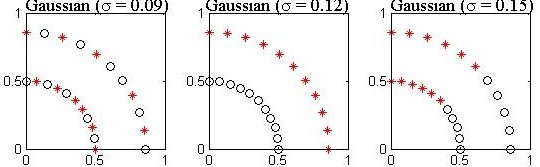}
\caption{(top row). Clustering obtained using MSC with $n$-point linear kernels for $n=2,7,20$, and 
(bottom row) results for Gaussian spectral clustering with $\sigma = 0.09, 0.12, 0.15$, respectively.}
\label{fig_lin_spectral}
\end{figure}

\begin{figure*}
\centering
\caption{Segmentation of images using spectral clustering with Gaussian and JT-kernels,  and MSC with $n$-point linear kernel.
Each row shows results for one image, and the best parameter value for each similarity is indicated
($q=1$ for JT denotes the JS-kernel).}
\label{fig_image}
\begin{tabular}{|c|c|c|c||c|c|c|c|}
\hline
\textbf{Original} &  \textbf{Gaussian} & \textbf{JT-kernel}	& \textbf{$n$-point}
&
\textbf{Original} &  \textbf{Gaussian} & \textbf{JT-kernel}	& \textbf{$n$-point}
\\
\textbf{image}	  &  \textbf{kernel} 	& 			& \textbf{linear}
&
\textbf{image}	  &  \textbf{kernel} 	& 			& \textbf{linear}
\\
\hline \hline
\includegraphics[width=\WidthOne]{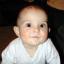} & 
\includegraphics[width=\WidthOne]{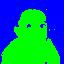} &
\includegraphics[width=\WidthOne]{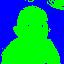} &
\includegraphics[width=\WidthOne]{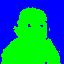}
&
\includegraphics[width=\WidthOne,height=\WidthOne]{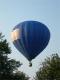}   & 
\includegraphics[width=\WidthOne,height=\WidthOne]{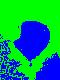} &
\includegraphics[width=\WidthOne,height=\WidthOne]{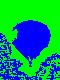} &
\includegraphics[width=\WidthOne,height=\WidthOne]{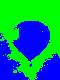}
\\
baby & $\sigma = 0.02$ & $q=1.25$ & $n=8$
&
balloon & $\sigma = 2.0$ & $q=1.0$ & $n=6$
\\
\hline \hline
\includegraphics[width=\WidthOne]{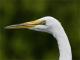}   & 
\includegraphics[width=\WidthOne]{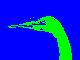} &
\includegraphics[width=\WidthOne]{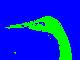} &
\includegraphics[width=\WidthOne]{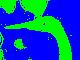}
&
\includegraphics[width=\WidthOne]{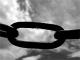}   & 
\includegraphics[width=\WidthOne]{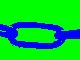} &
\includegraphics[width=\WidthOne]{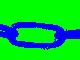} &
\includegraphics[width=\WidthOne]{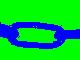}
\\
duck & $\sigma = 0.2$ & $q=1.25$ & $n=6$
&
chain & $\sigma = 0.02$ & $q=1.0$ & $n=10$
\\
\hline \hline
\includegraphics[width=\WidthOne,height=\WidthOne]{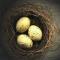} & 
\includegraphics[width=\WidthOne,height=\WidthOne]{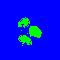} &
\includegraphics[width=\WidthOne,height=\WidthOne]{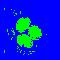} &
\includegraphics[width=\WidthOne,height=\WidthOne]{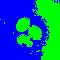}
&
\includegraphics[width=\WidthOne,height=\WidthOne]{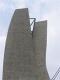} & 
\includegraphics[width=\WidthOne,height=\WidthOne]{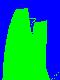} &
\includegraphics[width=\WidthOne,height=\WidthOne]{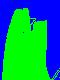} &
\includegraphics[width=\WidthOne,height=\WidthOne]{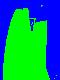}
\\
eggs & $\sigma = 0.02$ & $q=0.5$ & $n=8$
&
building & $\sigma = 0.2$ & $q=1.25$ & $n=8$
\\
\hline \hline
\includegraphics[width=\WidthOne]{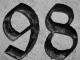}   & 
\includegraphics[width=\WidthOne]{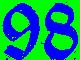} &
\includegraphics[width=\WidthOne]{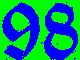} &
\includegraphics[width=\WidthOne]{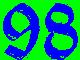}
&
\includegraphics[width=\WidthOne]{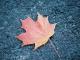}   & 
\includegraphics[width=\WidthOne]{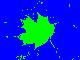} &
\includegraphics[width=\WidthOne]{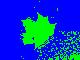} &
\includegraphics[width=\WidthOne]{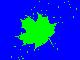}
\\
number & $\sigma = 0.02$ & $q=0.75$ & $n=6$
&
leaf & $\sigma = 0.2$ & $q=0.5$ & $n=6$
\\
\hline \hline
\includegraphics[width=\WidthOne,height=\WidthOne]{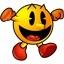} & 
\includegraphics[width=\WidthOne,height=\WidthOne]{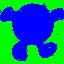} &
\includegraphics[width=\WidthOne,height=\WidthOne]{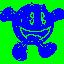} &
\includegraphics[width=\WidthOne,height=\WidthOne]{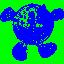}
&
\includegraphics[width=\WidthOne,height=\WidthOne]{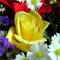} & 
\includegraphics[width=\WidthOne,height=\WidthOne]{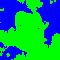} &
\includegraphics[width=\WidthOne,height=\WidthOne]{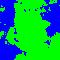} &
\includegraphics[width=\WidthOne,height=\WidthOne]{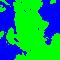}
\\
smiley & $\sigma = 2.0$ & $q=1.0$ & $n=12$
&
flowers & $\sigma = 2.0$ & $q=1.25$ & $n=10$
\\
\hline \hline
\includegraphics[width=\WidthOne]{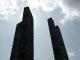}   & 
\includegraphics[width=\WidthOne]{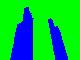} &
\includegraphics[width=\WidthOne]{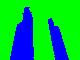} &
\includegraphics[width=\WidthOne]{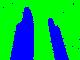}
&
\includegraphics[width=\WidthOne,height=\WidthOne]{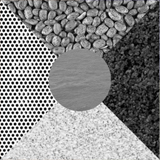} & 
\includegraphics[width=\WidthOne,height=\WidthOne]{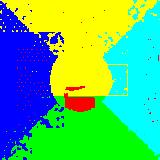} &
\includegraphics[width=\WidthOne,height=\WidthOne]{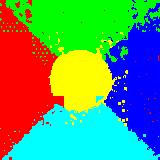} &
\includegraphics[width=\WidthOne,height=\WidthOne]{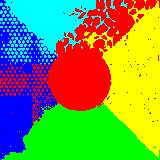}
\\
tower & $\sigma = 0.2$ & $q=1.25$ & $n=8$
&
texture & $\sigma = 2$ & $q=1.25$ & $n=12$
\\
\hline
\end{tabular}
\end{figure*}

\section{Experimental results}
\label{sec_experiments}
We compare the performance of MSC using the proposed
multi-point JT-kernels with Gaussian spectral clustering.
We do not compare our approach with methods proposed
in~\cite{Chen_2009_jour_JourCV,Govindu_2005_conf_CVPR} since these
methods require prior knowledge of geometric structures, and cannot be
applied to arbitrary data sets. We perform preliminary study on
standard data sets from~\cite{Frank_2010_misc_UCI}. 
Table~\ref{tab:multicol_spec} shows the accuracy of spectral clustering 
with Gaussian and 2-point JT-kernels, as well as MSC with
3-point JT-kernel and $n$-linear kernel.  
The performance measure considered is the purity of clusters obtained.
For 3-point JT-kernel, the $O(N^{n+1})$ unfolding method is used
since approximations give poor performance. For JT-kernels,
$q$ is varied over $[0,2]$ in steps of 0.25, where JS-kernel is used for the case $q=1$.
We also consider MSC with the $n$-point linear kernel, 
where we vary $n$ from 2 to 12 in steps of 2.
Similarly, for Gaussian case, we tune $\sigma$ to improve performance.
We note here that tuning $\sigma$ properly appeared to be more difficult 
than the parameters of our algorithms.
The accuracy is averaged over multiple runs of $k$-means step.
In Table~\ref{tab:multicol_spec}, we present the best accuracy results
achieved with each method. 
In general using 3-point JT-kernel is more effective,
though computationally extensive, which indicates that considering 
multiple points can improve performance.
The $n$-point linear kernels, which have reduced computational complexity
also perform quite well.

We use 2-point JT-kernel and $n$-point linear kernel for segmentation
of a number of images from~\cite{BCVG_misc_Datasets} and other sources (shown in Figure~\ref{fig_image}),
where each image is reduced to $60\times60$ or $80\times 60$.
We incorporate MSC and the proposed kernels
into the segmentation approach proposed in~\cite{Shi_2000_jour_TPAMI}.
For this,
we apply JT and $n$-point linear kernels on the pixel intensities 
(in \emph{texture} image, the intensities of the pixel and its neighbouring  8 pixels is quantized
into a histogram of 16 bins)
and compute the matrix $V$ using~\eqref{eq_lin_V}.
%
%
To make our method compatible with the partitioning algorithm, spectral clustering 
is performed using the affinity matrix
$M = V^{(1-\lambda)}R^{\lambda}$, where $\lambda=0.008$ and $R$ represents 
the similarity matrix for pixel locations computed as below.
If $p_i$ is location of the $i^{th}$ pixel, then
$R_{i,j} = e^{-\Vert p_{i}-p_{j} \Vert^{2}}$ if $\Vert p_{i}-p_{j} \Vert <r$,
and zero otherwise.
The idea is to partition 
the pixels into a number of clusters ($m=10$), and then group neighboring clusters
till we have desired segments, such that Ncut of the graph is minimized at each iteration.
Figure~\ref{fig_image} shows the best results for each of Gaussian, JT and $n$-point linear kernels
after tuning parameters. On an average, 
relative times of JT and $n$-point linear were 0.99 and 1.03, respectively, compared to Gaussian.
%
We observe that:
\\{\bf (1) }
JT-kernel with $q$ close to 1 (0.75--1.25) gives best results among
all values of $q$ in most cases. 
\\ {\bf (2) }
Though 2-point linear kernel gives poor results, 
as $n$ increases better partitions are obtained and mostly linear kernel
over $n=6$ to 10 points captures all necessary details.
\\ {\bf (3) }
On the whole, Figure~\ref{fig_image} makes it evident that in most cases,
JT and $n$-point kernel perform at par with Gaussian similarity,
and in fact, in some cases, JT-kernel shows significant
improvements (for instance, \emph{texture}, \emph{duck}, \emph{eggs}, \emph{building}).
The $n$-point linear kernel has a very simple structure, that of the dot-product,
and hence, is relatively poor. However, there are instances where it still outperforms
Gaussian (\emph{balloon}, \emph{texture}). 
To this end, Gaussian works significantly better than others only in \emph{smiley} image,
while segments in \emph{flowers} for all kernels are quite different.
\\ {\bf (4) }
In one case (\emph{flowers}), the segments obtained with JT and $n$-point
kernels are same, but this is different from that of Gaussian. 
This can be justified by the similar nature of JT to $n$-point (extension of JT with $q=2$),
that is significantly different from Gaussian.

One can note that similarities were constructed only using the pixel intensities and locations
as considered in~\cite{Shi_2000_jour_TPAMI}.
One can easily incorporate more sophisticated features into this setting,
and easily use JT or $n$-point linear kernels to evaluate their similarities. 
However, by construction and justifications given in Section~\ref{sec_real_extn}, the JT-kernel  
appears to be more applicable for histogram type of data such as pixel intensities.
On the other hand, Gaussian similarity is more applicable for 
pixel distances as used here.

\section{Discussions and concluding remarks}
\label{sec_conclusion}
We develop a spectral clustering (MSC) technique that uses a similarity among 
more than two points.
Our method is more general than existing algorithms of similar 
nature~\cite{Govindu_2005_conf_CVPR,Chen_2009_jour_JourCV}, as the algorithm 
does not depend on the similarity measure considered.
Though not discussed here, but one can easily incorporate out-of-sample extensions such
as Nystr\"{o}m's approximation to MSC.
To extend the idea further, it would be interesting to see if spectral methods
can be used on the similarity tensor, without unfolding it.

We also introduced the notion of multi-point kernels and
proposed a class of multi-point similarity measures that arise out of 
extension of positive definite two-point kernels. 
We also derived a multi-point extesion of linear kernels,
obtained for $q=2$ in multi-point JT-kernel that significantly simplifies computation 
of MSC algorithm. 
Toy examples similar to 
Figure~\ref{fig_lin_spectral} were studied, which revealed that $n$-point linear 
kernels were always able cluster accurately (above some $n$) when the clusters
are linearly separable. 
This promises a new direction of study: linear separability in the spectral clustering framework.


\section*{Acknowledgement}
D. Ghoshdastidar is supported by Google India Ph.D.
Fellowship in Statistical Learning Theory.

{\small
\bibliographystyle{ieee}
\bibliography{JT}
}

\end{document}